\documentclass[twoside,11pt]{article}
\usepackage[lmargin=1.25in,rmargin=1.25in,bmargin=1.25in,tmargin=1.25in, headsep=.5in]{geometry}
\usepackage[margin=0in]{caption}

\usepackage{authblk}
\usepackage{amsmath}
\usepackage{amsthm}
\usepackage{amsfonts}
\usepackage{times}
\usepackage{graphicx}
\usepackage{color}
\usepackage{rotating}
\usepackage{bbm}
\usepackage{latexsym}
\usepackage{dsfont}
\usepackage{enumerate}
\usepackage{enumitem}
\usepackage{verbatim}
\usepackage{slashbox,booktabs,array}

\usepackage{natbib}

\usepackage{wasysym}
\usepackage{tabu}

\definecolor{bl}{RGB}{20,20,200}
\usepackage{hyperref}
\hypersetup{ colorlinks=true,  citecolor=bl,  linkcolor=bl,  urlcolor=bl}

\newtheorem{theorem}{Theorem}
\newtheorem{lemma}[theorem]{Lemma}
\newtheorem{corollary}[theorem]{Corollary}
\newtheorem{proposition}[theorem]{Proposition}

\theoremstyle{remark}

\theoremstyle{definition}
\newtheorem{example}[theorem]{Example}
\newtheorem{definition}[theorem]{Definition}

\newcommand{\conv}{\ensuremath{\operatorname{conv}}}

\newcommand{\Acal}{\mathcal{A}}
\newcommand{\Lcal}{\mathcal{L}}
\newcommand{\Ecal}{\mathcal{E}}
\newcommand{\Mcal}{\mathcal{M}}

\newcommand{\Pcal}{\mathcal{P}}

\newcommand{\Xcal}{\mathcal{X}}

\newcommand{\Ycal}{\mathcal{Y}}

\newcommand{\Ccal}{\mathcal{C}}

\newcommand{\Zcal}{\mathcal{Z}}
\newcommand{\Fcal}{\mathcal{F}}
\newcommand{\R}{\mathbb{R}}

\newcommand{\Amd}{\mathfrak{A}}
\newcommand{\Kmd}{\mathfrak{K}}

\newcommand{\supp}{\operatorname{supp}}

\newcommand{\TRBM}{\ensuremath{\operatorname{RBM}^{\text{\normalfont tropical}}}}
\newcommand{\RBM}{\ensuremath{\operatorname{RBM}}}

\newcommand{\be}{\boldsymbol{\operatorname{e}}}

\usepackage{fancyhdr}
\begin{document}

\thispagestyle{empty}
\title{\Large{\bf Discrete Restricted Boltzmann Machines}}  

\author{ 
\vspace{.3cm}
\small{\bf Guido Mont\'ufar} \hfill \small{\textsc{gfm10@psu.edu}} \\
\vspace{-.2cm}
\small{\bf Jason Morton} \hfill \small{\textsc{morton@math.psu.edu}} \\
\vspace{.4cm}
\small{\emph{Department of Mathematics}}\hfill \mbox{ }\\
\small{\emph{Pennsylvania State University}}\hfill\mbox{ }\\ 
\small{\emph{University Park, PA 16802, USA}}\hfill\mbox{ }
\vspace{.4cm}
}

\date{ }

\maketitle

\begin{abstract}%   
\noindent We describe discrete restricted Boltzmann machines: probabilistic graphical models with bipartite interactions between visible and hidden discrete variables. 
Examples are binary restricted Boltzmann machines and discrete na\"ive Bayes models. 
We detail the inference functions and distributed representations arising in these models in terms of configurations of projected products of simplices and normal fans of products of simplices. 
We bound the number of hidden variables, depending on the cardinalities of their state spaces, for which these models can approximate any probability distribution on their visible states to any given accuracy. 
In addition, we use algebraic methods and coding theory to compute their dimension. 

\medskip

\end{abstract}

\section{Introduction}

A restricted Boltzmann machine (RBM) is a probabilistic graphical model with bipartite interactions between an observed set and a hidden set of units~\citep[see][]{Smolensky1986,Freund1992,Hinton2002,Hinton2010}. 
A characterizing property of these models is that the 
observed units are independent given the states of the hidden units and vice versa. 
This is a consequence of the bipartiteness of the interaction graph and does not depend on the units' state spaces. 
Typically RBMs are defined with binary units, 
but other types of units have also been considered, including continuous, discrete, and mixed type units~\citep[see][]{welling:exponential,Marks01diffusionnetworks,Salakhutdinov:2007,wordrbm,mixvarrbm}. 
We study discrete RBMs, also called multinomial or softmax RBMs, which are special types of exponential family harmoniums~\citep[][]{welling:exponential}. 
While each unit $X_i$ of a binary RBM has the state space $\{0,1\}$, the  state space of each unit $X_i$ of a discrete RBM is a finite set $\Xcal_i=\{0,1,\ldots, r_i-1\}$. 
Like binary RBMs, discrete RBMs can be trained using contrastive divergence (CD)~\citep{Hinton99productsof,Hinton2002,Carreira2005} or expectation-maximization (EM)~\citep{dempsterEM} and can be used to train the parameters of deep systems layer by layer~\citep{Hinton2006,Bengio2007}. 

Non-binary visible units are natural because they can directly encode non-binary features. 
The situation with hidden units is more subtle. 
States that appear in different hidden units can be activated by the same visible vector, but states that appear in the same hidden unit are mutually exclusive. 
Non-binary hidden units thus allow one to explicitly represent complex exclusive relationships. 
For example, a discrete RBM topic model would allow some topics 
to be mutually exclusive and other topics to be mixed together freely. 
This provides a better match to the semantics of several learning problems, although the learnability of such representations is mostly open. 
The practical need to represent mutually exclusive properties is evidenced by the common approach of adding activation sparsity parameters to binary RBM hidden states, which artificially create mutually exclusive non-binary states by penalizing models which have more than a certain percentage of hidden units active. 

A discrete RBM is a {\em product of experts}~\citep{Hinton99productsof}; each hidden unit represents an expert which is a mixture model of product distributions, or na\"ive Bayes model. 
Hence discrete RBMs capture both na\"ive Bayes models and binary RBMs,  %, which are discrete RBMs with one single hidden unit. 
and interpolate between non-distributed %unrestricted 
mixture representations % from na\"ive Bayes models. 
and distributed %restricted 
mixture representations~\citep{Bengio-2009, MontufarMorton2012}. 
See Figure~\ref{RBMhierarchyfig}. 
\begin{figure}
\setlength{\unitlength}{14cm}
\centering
\begin{picture}(.9,.18)(0,0)
\put(.04,0){\includegraphics[trim=3.7cm 21cm 14.5cm 2cm, clip=true, scale=1.2]{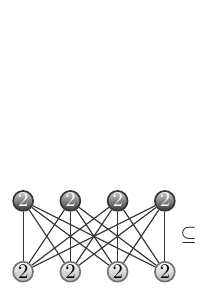}}
\put(.33,0){\includegraphics[trim=10.4cm 21cm 4.7cm 2cm, clip=true, scale=1.2]{RBMhierarchy2.pdf}}
\put(.058,.17){\begin{minipage}{.24\unitlength}\center $24$\end{minipage}}%{$mn + m + n$}
\put(.363,.17){\begin{minipage}{.2\unitlength}\center $40$\end{minipage}}%{$(k_1-1) n + (k_2-1)n+n+(k_1-1) +(k_2-1) $}
\put(.645,.17){\begin{minipage}{.2\unitlength}\center $79$\end{minipage}}%{$(k-1) n +n + k-1$}
\end{picture}
\caption{Examples of probability models treated in this paper, in the special case of binary visible variables. 
The light (dark) nodes represent visible (hidden) variables with the indicated number of states. 
%The number inside each node indicates the states of the corresponding variable. 
%Each variable takes any particular state in its state space with a probability that depends on the states of the adjacent variables (the nodes it is connected to  by an edge) and the interaction parameters. 
The total parameter count of each model is indicated at the top. 
From left to right: a binary RBM; 
a discrete RBM with one $8$-valued and one binary hidden units; and a binary na\"ive Bayes model with $16$ hidden classes. 
}
\label{RBMhierarchyfig}
\end{figure}
Na\"ive Bayes models have been studied across many disciplines. In machine learning they are most commonly used for classification and clustering, but have also been considered for probabilistic modelling~\citep{Lowd05naivebayes,Montufar2010a}. 
Theoretical work on binary RBM models includes results on universal approximation~\citep{Freund1992, LeRoux2008,  Montufar2011}, dimension and parameter identifiability~\citep{Cueto2010}, Bayesian learning coefficients~\citep{aoyagi:2010}, complexity~\citep{LongServedio10}, and approximation errors~\citep{NIPS2011_0307}. 
In this paper we generalize some of these theoretical results to discrete RBMs.

Probability models with more general interactions than strictly bipartite have also been considered, including semi-restricted Boltzmann machines and higher-order interaction Boltzmann machines~\citep[see][]{Sejnowski86higher-orderboltzmann,Memisevic2010,DBLP:conf/nips/OsinderoH07,RanzatoKH10}. The techniques that we develop in this paper also serve to treat a general class of RBM-like models allowing within-layer interactions, a generalization that will be carried out in a forthcoming work~\citep{MontufarMorton2013b}. 

\medskip 

Section~\ref{section:preliminaries} collects basic facts about independence models, na\"ive Bayes models, and binary RBMs, including an overview on the aforementioned theoretical results. 
Section~\ref{section:multRBM} defines discrete RBMs formally 
and describes them as (i) products of mixtures of product distributions (Proposition~\ref{prodmixt}) and (ii) as restricted mixtures of product distributions. Section~\ref{section:inference} elaborates on distributed representations and inference functions represented by discrete RBMs (Proposition~\ref{paralelslicings}, Lemma~\ref{strmod1}, and Proposition~\ref{strongmodepro}). 
Section~\ref{section:expressive} addresses the expressive power of discrete RBMs by describing explicit submodels (Theorem~\ref{corpropouniv}) and provides results on their maximal approximation errors and universal approximation properties (Theorem~\ref{approxerrordiscreteRBM}).  
Section~\ref{section:algebraic} treats the dimension of discrete RBM models (Proposition~\ref{remdimu} and  Theorem~\ref{theorem:dimension}). Section~\ref{section:tropical model} contains an algebraic-combinatorial discussion of tropical discrete RBM models (Theorem~\ref{tropicalRBM}) with consequences for their dimension collected in Propositions~\ref{corbinhi},~\ref{binhidcor}, and~\ref{generalcasecor}.

\section{Preliminaries}\label{section:preliminaries}

\subsection{Independence models}
Consider a system of $n<\infty$  random variables $X_1,\ldots,X_n$. Assume that $X_i$ takes states $x_i$ in a finite set $\Xcal_i=\{0,1,\ldots,r_i-1\}$  for all $i\in\{1,\ldots,n\}=:[n]$. The state space of this system is $\Xcal:=\Xcal_1\times\cdots\times\Xcal_n$. We write $x_\lambda=(x_i)_{i\in\lambda}$ for a joint state of the variables with index $i\in\lambda$ for any $\lambda\subseteq[n]$, and  $x=(x_1,\ldots,x_n)$ for a joint state of all variables. 
We denote by $\Delta(\Xcal)$ the set of all probability distributions on $\Xcal$. 
We write $\langle a ,b\rangle$ for the inner product $a^\top b$. 

The {\em independence model} of the variables $X_1,\ldots, X_n$ is the set of product distributions $p(x) = \prod_{i\in[n]} p_i(x_i)$ for all $x\in\Xcal$, where $p_i$ is a probability distribution with state space $\Xcal_i$ for all $i\in[n]$. This model is the closure $\overline{\Ecal_\Xcal}$ (in the Euclidean topology) of the exponential family 
\begin{equation}
\mathcal{E}_\Xcal := \Big\{ \frac{1}{Z(\theta)} \exp(\langle\theta,A^{(\Xcal)}\rangle) \colon \theta\in\R^{d_\Xcal}\Big\} ,
\end{equation}
where $A^{(\mathcal{X})}\in\R^{d_\Xcal\times\Xcal}$ is a matrix of sufficient statistics; with rows equal to the indicator functions $\mathds{1}_\Xcal$ and $\mathds{1}_{\{x\colon x_i=y_i\}}$ for all $y_i\in\Xcal_i\setminus\{0\}$ for all $i\in[n]$. The partition function $Z(\theta)$ normalizes the distributions. 
The convex support of $\Ecal_\Xcal$ is the convex hull $Q_\Xcal:=\conv(\{A^{(\Xcal)}_x\}_{x\in\Xcal})$ of the columns of $A^{(\Xcal)}$, which is a Cartesian product of simplices with $Q_\Xcal \cong \Delta(\Xcal_1)\times \cdots\times \Delta(\Xcal_n)$. 

\begin{figure}
\setlength{\unitlength}{\textwidth}
\begin{center}
\begin{picture}(.65,.26)(0,.02)
\put(0,0){\includegraphics[trim=4cm 20cm 9cm 2cm, clip=true, width=.6\unitlength]{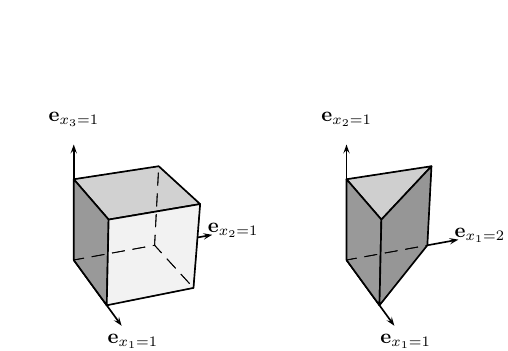}}
\end{picture}
\end{center}
\caption{The convex support of the independence model of three binary variables (left) and  of a binary-ternary pair of variables (right) discussed in Example~\ref{exampleindepcs}. 
}
\label{Boxesfig}
\end{figure}

\begin{example}\label{exampleindepcs}
\begin{normalfont}
The sufficient statistics of the independence models $\Ecal_\Xcal$ and $\Ecal_{\Xcal'}$ with 
state spaces $\Xcal=\{0,1\}^3$ and $\Xcal'=\{0,1,2\}\times\{0,1\}$ are, with rows labeled by indicator functions, 
\newcommand{\bm}{\!\!\!\!\begin{bmatrix}}
\newcommand{\enm}{\end{bmatrix}\!\!\!\!}
\newcolumntype{T}{>{\scriptsize$}c<{$}}
\begin{equation*}
A^{(\Xcal)}=
\overset{
\renewcommand{\arraystretch}{0.5}
\begin{tabular}{T T T T T T T T }
\renewcommand{\arraystretch}{.7}
\bm 1\\1\\1 \enm &
\renewcommand{\arraystretch}{.7}
\bm 1\\1\\0 \enm &
\renewcommand{\arraystretch}{.7}
\bm 1\\0\\1 \enm &
\renewcommand{\arraystretch}{.7}
\bm 1\\0\\0 \enm &
\renewcommand{\arraystretch}{.7}
\bm 0\\1\\1 \enm &
\renewcommand{\arraystretch}{.7}
\bm 0\\1\\0 \enm &
\renewcommand{\arraystretch}{.7}
\bm 0\\0\\1 \enm &
\renewcommand{\arraystretch}{.7}
\bm 0\\0\\0 \enm \\ 
& & & & & & &
\end{tabular}
}{
\renewcommand{\arraystretch}{.8}
\left(\begin{array}{c c c c c c c c}
1&1&1&1&1&1&1&1\\\midrule[0.02em]
1&1&1&1&0&0&0&0\\
1&1&0&0&1&1&0&0\\ 
1&0&1&0&1&0&1&0\\ 
\end{array}\right)
}\!\!
\overset{
\renewcommand{\arraystretch}{0.5}
\begin{tabular}{T}
\renewcommand{\arraystretch}{0.75}
\phantom{\bm 1\\1\\1 \enm}\\
~\\ 
\end{tabular}
}{
\renewcommand{\arraystretch}{.8}
\begin{array}{l}
~ \\\midrule[0em]
x_3=1\\ 
x_2=1\\ 
x_1=1
\end{array}
}
\;\quad 
A^{(\Xcal')}=
\overset{
\renewcommand{\arraystretch}{0.5}
\begin{tabular}{T T T T T T}
\renewcommand{\arraystretch}{0.7}
\bm 1\\2 \enm &
\renewcommand{\arraystretch}{0.7}
\bm 1\\1 \enm &
\renewcommand{\arraystretch}{0.7}
\bm 1\\0 \enm &
\renewcommand{\arraystretch}{0.7}
\bm 0\\2 \enm &
\renewcommand{\arraystretch}{0.7}
\bm 0\\1 \enm &
\renewcommand{\arraystretch}{0.7}
\bm 0\\0 \enm\\ 
& & & & & 
\end{tabular}
}{
\renewcommand{\arraystretch}{.8}
\left(\begin{array}{c c c | c c c}
 1 &1 &1 & 1&1 &1 \\\midrule[0.02em]
 1 &1 &1 &0 &0 &0 \\\midrule[0.02em]
 1 &0 &0 &1 &0 &0 \\ 
 0 &1 &0 &0 &1 &0 
\end{array}
\right)
}\!\!
\overset{
\renewcommand{\arraystretch}{0.5}
\begin{tabular}{T}
\renewcommand{\arraystretch}{0.75}
\phantom{\bm 1\\1 \enm}\\
~ \\ 
\end{tabular}
}{
\renewcommand{\arraystretch}{.8}
\begin{array}{l}
~ \\\midrule[0em]
x_2=1\\\midrule[0em]
x_1=2\\ 
x_1=1
\end{array}
}. 
\end{equation*}
In the first case the convex support is a cube and in the second it is a prism. 
Both convex supports are three-dimensional polytopes, but the prism has fewer vertices and is more similar to a simplex, meaning that its vertex set is affinely more independent than that of the cube. 
See Figure~\ref{Boxesfig}. 
\end{normalfont}
\end{example}

\subsection{Na\"ive Bayes models}\label{nBm}

Let $k\in\mathbb{N}$. The {\em $k$-mixture} of the independence model, or {\em na\"ive Bayes model} with $k$ hidden classes, with visible variables $X_1,\ldots,X_n$ 
is the set of all probability distributions expressible as convex combinations of $k$ points in $\Ecal_\Xcal$: 
\begin{equation}
%\RBM_{\Xcal,\Ycal_1}=
\Mcal_{\Xcal,k}:=\Big\{\sum_{i\in[k]} \lambda_i p^{(i)}\colon p^{(i)}\in\Ecal_\Xcal,\; \lambda_i\geq0, \text{ for all } i\in[k],\text{ and } \sum_{i\in[k]}\lambda_i=1\Big\} .
\end{equation} 

We write  $\Mcal_{n,{k}}$ for the $k$-mixture of the independence model of $n$ binary variables. 
The dimensions of mixtures of binary independence models are known: 
\begin{theorem}[\citet{Catalisano2011}]\label{Catalisano}
The mixtures of binary independence models $\Mcal_{n,k}$ have the dimension expected from counting parameters, $\min\{nk + (k-1),2^n-1\}$, except for $\Mcal_{4,3}$, which has dimension $13$ instead of $14$. 
\end{theorem}

Let $\Amd_\Xcal(d)$ denote  the maximal cardinality of a subset $\Xcal'\subseteq\Xcal$ of minimum Hamming distance at least $d$, i.e., the maximal cardinality of a subset $\Xcal'\subseteq\Xcal$ with $d_H(x,y)\geq d$ for all distinct points $x,y\in\Xcal'$, where $d_H(x,y):=|\{i\in[n]\colon x_i\neq y_i\}|$ denotes the Hamming distance between $x$ and $y$. 
The function $\Amd_\Xcal$ is familiar in coding theory. 
The $k$-mixtures of independence models are universal approximators when $k$ is large enough. This can be made precise in terms of $\Amd_\Xcal(2)$: 

\begin{theorem}[{\citet{Montufar2010a}}]\label{unimixt}
The mixture model $\Mcal_{\Xcal,k}$ can approximate any probability distribution on $\Xcal$ arbitrarily well if $k\geq  {|\Xcal|}/{\max_{i\in[n]}|\Xcal_i|}$ and 
only if $k \geq \Amd_\Xcal(2)$. 
\end{theorem}

By results from~\citep[][]{Gilbert:1952,Varshamov:1957}, when $q$ is a power of a prime number and $\Xcal=\{0,1,\ldots,q-1\}^n$, then $\Amd_\Xcal=q^{n-1}$. In these cases the previous theorem shows that %Theorem~\ref{unimixt} 
$\Mcal_{\Xcal,k}$ is a universal approximator of distributions on $\Xcal$ if and only if $k\geq q^{n-1}$. 
In particular, the smallest na\"ive Bayes model universal approximator of distributions on $\{0,1\}^n$ has $2^{n-1}(n+1)-1$ parameters. 

Some of the distributions not representable by a given na\"ive Bayes model can be characterized in terms of their modes. 
A state $x\in\Xcal$ is a {\em mode} of a distribution $p\in\Delta(\Xcal)$ if $p(x)>p(y)$ for all $y$ with $d_H(x,y)=1$ and it is a {\em strong mode}\label{strongmodepage} if $p(x)>\sum_{y\colon d_H(x,y)=1} p(y)$. 

\begin{lemma}[\citet{MontufarMorton2012}]\label{strongmodeslemma}
If a mixture of product distributions  $p=\sum_i \lambda_i p^{(i)}$ has strong modes  $\Ccal\subseteq\Xcal$, then there is a mixture component $p^{(i)}$ with mode $x$ for each $x\in\Ccal$. 
\end{lemma}

\subsection{Binary restricted Boltzmann machines}

The binary RBM model with $n$ visible and $m$ hidden units, denoted $\RBM_{n,m}$, 
is the set of distributions on $\{0,1\}^n$ of the form 
\begin{equation}
p(x) = \frac{1}{Z(W,B,C)} \sum_{h\in\{0,1\}^m}\exp(h^\top W x + B^\top x + C^\top h)\quad\text{for all } x\in\{0,1\}^n , \label{binrbmdef}
\end{equation} 
where $x$ denotes states of the visible units, $h$ denotes states of the hidden units, 
$W = (W_{ji})_{ji}\in \R^{m\times n}$ is a matrix of interaction weights, 
$B \in \R^n$ and $C \in \R^m$ are vectors of bias weights, and $Z(W,B,C)=\sum_{x\in\{0,1\}^n}\sum_{h\in\{0,1\}^m}\exp(h^\top W x + B^\top x + C^\top h)$ is the normalizing partition function.

It is known that these models have the expected dimension for many choices of $n$ and $m$: 
\begin{theorem}[\citet{Cueto2010}] 
The dimension of the model $\RBM_{n,m}$ is equal to $nm+n+m$ when $m+1\leq 2^{n-\lceil\log_2(n+1)\rceil}$ and it is equal to $2^n - 1$ when $m\geq 2^{n-\lfloor\log_2(n+1)\rfloor}$. 
\end{theorem}

It is also known that with enough hidden units, binary RBMs are universal approximators: 

\begin{theorem}[\citet{Montufar2011}]\label{universalbinRBM}
The model $\RBM_{n,m}$ can approximate any distribution on $\{0,1\}^n$ arbitrarily well whenever $m\geq2^{n-1}-1$. 
\end{theorem}

A previous result by~\citet[][Theorem~2]{LeRoux2008} shows that $\RBM_{n,m}$ is a universal approximator whenever $m\geq 2^n+1$. 
It is not known whether the bounds from Theorem~\ref{universalbinRBM} 
are always tight, but they show that for any given $n$, the smallest RBM universal approximator of distributions on $\{0,1\}^n$ has at most $2^{n-1}(n+1)-1$ parameters and hence not more than the smallest na\"ive Bayes model universal approximator (Theorem~\ref{unimixt}).

\section{Discrete restricted Boltzmann machines}\label{section:multRBM} 

Let $\Xcal_i=\{0,1,\ldots,r_i-1\}$ for all $i\in[n]$ and $\Ycal_j=\{0,1,\ldots,s_j-1\}$ for all $j\in[m]$. 
The graphical model with full bipartite interactions $\{\{i,j\}\colon i\in[n], j\in[m]\}$ on  $\Xcal\times\Ycal$ is the exponential family 
\begin{equation}
\Ecal_{\Xcal,\Ycal}:=\left\{ \frac{1}{Z(\theta)}\exp(\langle \theta, A^{(\Xcal,\Ycal)}\rangle)\colon \theta\in\R^{d_\Xcal d_\Ycal}\right\},
\label{exponentialfamilydef}
\end{equation}
with sufficient statistics matrix equal to the Kronecker product $A^{(\Xcal,\Ycal)}=A^{(\Xcal)}\otimes A^{(\Ycal)}$ of the sufficient statistics matrices $A^{(\Xcal)}$ and $A^{(\Ycal)}$ of the independence models $\Ecal_\Xcal$ and $\Ecal_\Ycal$. 
%\begin{definition}\begin{normalfont}
The matrix $A^{(\Xcal,\Ycal)}$ has $d_\Xcal d_\Ycal=\left(\sum_{i\in[n]} (|\Xcal_i|-1) +1\right)\left(\sum_{j\in[m]} (|\Ycal_i|-1) +1\right)$ linearly independent rows and  $|\Xcal\times\Ycal|$ columns, each column corresponding to a joint state $(x,y)$ of all variables. 
Disregarding the entry of $\theta$ that is multiplied with the constant row of $A^{(\Xcal,\Ycal)}$, which cancels out with the normalization function $Z(\theta)$, this parametrization of $\Ecal_{\Xcal,\Ycal}$ is one-to-one. In particular, this model has dimension $\dim( \Ecal_{\Xcal,\Ycal}) = d_\Xcal d_\Ycal-1$. 

\medskip 

The discrete RBM model $\RBM_{\Xcal,\Ycal}$ is the following set of marginal distributions:  
\begin{equation}
\RBM_{\Xcal,\Ycal}:=\Big\{q(x)=\sum_{y\in\Ycal}p(x,y) \text{ for all } x\in\Xcal \colon p\in\Ecal_{\Xcal,\Ycal} \Big\}.
\end{equation}

In the case of one single hidden unit, this model is the na\"ive Bayes model on $\Xcal$ with $|\Ycal_1|$ hidden classes. 
When all units are binary, $\Xcal=\{0,1\}^n$ and $\Ycal=\{0,1\}^m$, this model is $\RBM_{n,m}$. 
Note that the exponent in eq.~\eqref{binrbmdef} can be written as $(h^\top W x +B^\top x + C^\top h) =  \langle \theta, A^{(\Xcal,\Ycal)}_{(x,h)} \rangle$, 
taking for $\theta$ the column-by-column vectorization of the matrix $\bigl(\begin{smallmatrix} 0 & B^\top\\ C & W\end{smallmatrix}\bigr)$.

\subsection*{Conditional distributions}
The conditional distributions of discrete RBMs can be described in the following way. 
Consider a vector $\theta\in\R^{d_\Xcal d_\Ycal}$ parametrizing $\Ecal_{\Xcal,\Ycal}$, and the matrix $\Theta\in\R^{d_\Ycal\times d_\Xcal}$ with column-by-column vectorization equal to $\theta$. A lemma by~\citet{Roth1934} shows that 
$\theta^\top (A^{(\Xcal)}\otimes A^{(\Ycal)})_{(x,y)} =  (A^{(\Xcal)}_x)^\top \Theta^\top A^{(\Ycal)}_y$ for all $x\in\Xcal$, $y\in\Ycal$, and hence 
\begin{gather}
\left\langle\theta, A^{(\Xcal,\Ycal)}_{(x,y)}\right\rangle = \left\langle \Theta A^{(\Xcal)}_x, A^{(\Ycal)}_y\right\rangle= \left\langle \Theta^\top  A^{(\Ycal)}_y, A^{(\Xcal)}_x \right\rangle\quad\forall x\in\Xcal, y\in\Ycal . \label{eq:Roth}
\end{gather}  

The inner product in eq.~\eqref{eq:Roth} describes following probability distributions: 
\begin{eqnarray}
p_\theta(\cdot,\cdot) &=& \frac{1}{Z(\theta)} \exp\big(\big\langle\theta, A^{(\Xcal,\Ycal)} \big\rangle\big),\\
p_\theta(\cdot|x) &=& \frac{1}{Z\big(\Theta A^{(\Xcal)}_x\big)} \exp\big(\big\langle \Theta A^{(\Xcal)}_x, A^{(\Ycal)} \big\rangle\big), \text{ and }\\
p_\theta(\cdot|y) &=& \frac{1}{Z\big(\Theta^\top  A^{(\Ycal)}_y\big)} \exp\big(\big\langle \Theta^\top  A^{(\Ycal)}_y, A^{(\Xcal)}  \big\rangle\big).
\end{eqnarray}
Geometrically, $\Theta A^{(\Xcal)}$ is a linear projection of the columns of the sufficient statistics matrix $A^{(\Xcal)}$ into the parameter space of $\Ecal_\Ycal$, and similarly, $\Theta^\top  A^{(\Ycal)}$ is a linear projection of the columns of $A^{(\Ycal)}$ into the parameter space of $\Ecal_\Xcal$.

\subsection*{Polynomial parametrization}
Discrete RBMs can be parametrized not only in the exponential way discussed above, 
but also by simple polynomials. 
The exponential family $\Ecal_{\Xcal,\Ycal}$ can be parametrized by square free monomials:  
\begin{equation}
p(v,h) = \frac1Z \prod_{\scriptsize\begin{matrix}{\{j,i\}\in[m]\times[n],}\\{ (y_j',x_i')\in\Ycal_j\times\Xcal_i}\end{matrix}} (\gamma_{\{j,i\},(y_j',x_i')})^{\delta_{y_j'}(h_j)\delta_{x_i'}(v_i)}\; \text{ for all } (v,h)\in\Ycal\times\Xcal ,  
\end{equation}
where $\gamma_{\{j,i\},(y_j',x_i')}$ are positive reals. %\in\R_>$. 
%$\gamma_{\{j,i\},(y_j',x_i')}=\exp(\theta_{\{j,i\},(y_j',x_i')})$ 
The probability distributions in $\RBM_{\Xcal,\Ycal}$ can be written as 
\begin{equation}
p(v) = \frac1Z \prod_{j\in[m]} \Big( \sum_{h_j\in\Ycal_j} \gamma_{\{j,1\},(h_j,v_1)}\cdots \gamma_{\{j,n\},(h_j,v_n)} \Big)\quad \text{ for all } v\in\Xcal.
\label{polynomialRBMmap}
\end{equation} 
The parameters $\gamma_{\{j,i\},(y_j',x_i')}$ correspond to $\exp(\theta_{\{j,i\},(y_j',x_i')})$ in the parametrization given in eq.~\eqref{exponentialfamilydef}.

\subsection*{Products of mixtures and mixtures of products}\label{section:distributed}

In the following we describe discrete RBMs from two complementary perspectives: 
(i) as products of experts, where each expert is a mixture of products, and (ii) as restricted mixtures of product distributions.  
The renormalized entry-wise (Hadamard) product of two probability distributions $p$ and $q$ on $\Xcal$ is  
defined as $p\circ q :=(p(x)q(x))_{x\in\Xcal}/ \sum_{y\in\Xcal} p(y)q(y)$. 
Here we assume that $p$ and $q$ have overlapping supports, such that the definition makes sense. 
%The product of two models $\Mcal,\Mcal'\subseteq\Delta(\Xcal)$ is the set $\Mcal\circ\Mcal'=\{p\circ q\colon p\in\Mcal, q\in\Mcal'\}$. 

\begin{proposition}\label{prodmixt}
The model $\RBM_{\Xcal,\Ycal}$ is a Hadamard product of mixtures of product distributions: 
\begin{equation*}
\RBM_{\Xcal,\Ycal} = \Mcal_{\Xcal,|\Ycal_1|}\circ\cdots \circ \Mcal_{\Xcal,|\Ycal_m|}\,.
\end{equation*}
\end{proposition}
\begin{proof}\,  
%Proposition \ref{prodmixt} 
The statement can be seen directly by considering the parametrization from eq.~\eqref{polynomialRBMmap}.  To make this explicit, one can use a {\em homogeneous} version of the matrix $A^{(\Xcal,\Ycal)}$ which we denote by $A$ and which defines the same model.  
Each row of $A$ is indexed by an edge $\{i,j\}$ of the bipartite graph and a joint state $(x_i, h_j)$ of the visible and hidden units connected by this edge. 
Such a row has a one in any column when these states agree with the global state, and zero otherwise.  
For any $j\in[m]$ let $A_{j,:}$ denote the matrix containing the rows of $A$ with indices 
$(\{i,j\}, (x_i,h_j))$ for all $x_i\in\Xcal_i$ for all $i\in[n]$ for all $h_j\in\Ycal_j$, and let $A(x,h)$ denote the $(x,h)$-column of $A$. We have
\begin{align*}
p(x)=&\frac1Z \sum_h \exp(\langle \theta, A(x,h)\rangle) \\
=& \frac1Z \sum_h \exp(\langle \theta_{1,:}, A_{1,:}(x,h)\rangle) \exp(\langle \theta_{2,:}, A_{2,:}(x,h)\rangle)\cdots \exp(\langle \theta_{m,:}, A_{m,:}(x,h)\rangle) \\
=& \frac1Z \Big(\sum_{h_1} \exp(\langle \theta_{1,:}, A_{1,:}(x,h_1)\rangle) \Big) 
\cdots \Big(\sum_{h_m}\exp(\langle \theta_{m,:}, A_{m,:}(x,h_m)\rangle)\Big) \\
=&\frac1Z (Z_1 p^{(1)}(x)) \cdots (Z_m p^{(m)}(x)) %\\=&
=\frac1{Z'} p^{(1)}(x)\cdots p^{(m)}(x) , 
\end{align*}
where $p^{(j)}\in\Mcal_{\Xcal,|\Ycal_j|}$ and $Z_j=\sum_{x\in\Xcal} \sum_{h_j\in\Ycal_j}\exp(\langle \theta_{j,:}, A_{j,:}(x,h_j)\rangle)$ for all $j\in[m]$. 
Since the vectors $\theta_{j,:}$ can be chosen arbitrarily, the factors $p^{(j)}$ can be made arbitrary within $\Mcal_{\Xcal,|\Ycal_j|}$. 
\end{proof}

Of course, every distribution in $\RBM_{\Xcal,\Ycal}$ 
is a mixture distribution $p(x) = \sum_{h\in\Ycal} p(x|h) q(h)$. 
The mixture weights are given by the marginals $q(h)$ on $\Ycal$ of distributions from $\Ecal_{\Xcal,\Ycal}$, and the mixture components can be described as follows.  

\begin{proposition}\label{discrRBMmixtcomps}
The set of conditional distributions $p(\cdot|h)$, $h\in\Ycal$ of a distribution in $\Ecal_{\Xcal,\Ycal}$ is the set of product distributions in $\Ecal_\Xcal$ with parameters 
$\theta_h =\Theta^\top A^{(\Ycal)}_h$, $h\in\Ycal$ equal to a linear projection of the vertices $\{A^{(\Ycal)}_h\colon h\in\Ycal\}$ of the Cartesian product of simplices $Q_{\Ycal} \cong \Delta( \Ycal_1)\times\cdots\times \Delta( \Ycal_m)$. 
\end{proposition}
\begin{proof}\,  
This is by eq.~\eqref{eq:Roth}. 
\end{proof}

\section{Products of simplices and their normal fans}\label{section:inference}

Binary RBMs have been analyzed by considering each of the $m$ hidden units as defining a hyperplane $H_j$ slicing the $n$-cube into two regions.  To generalize the results provided by this analysis, in this section we replace the $n$-cube with a general product of simplices $Q_{\Xcal}$, and replace the two regions defined by the hyperplane $H_j$ by the $|\Ycal_j|$ regions defined by the maximal cones of the normal fan of the simplex $\Delta ({\Ycal_j})$.

\subsection*{Subdivisions of independence models}

The {\em normal cone} of a polytope $Q\subset\R^d$ at a point $x\in Q$ is the set of all vectors $v\in\R^d$ with $\langle v, (x-y)\rangle\geq0$ for all $y\in Q$. 
We denote by $R_x$ the normal cone of the product of simplices $Q_\Xcal=\conv\{A^{(\Xcal)}_x\}_{x\in\Xcal}$ at the vertex $A^{(\Xcal)}_x$. 
The  normal fan  $\Fcal_\Xcal$ is the set of all normal cones of $Q_\Xcal$. 
The product distributions  $p_\theta = \frac{1}{Z(\theta)}\exp(\langle\theta,A^{(\Xcal)} \rangle) \in\Ecal_\Xcal$ strictly maximized at $x\in\Xcal$, with $p_\theta(x)> p_\theta(y)$ for all $y\in\Xcal\setminus\{x\}$, are those with parameter vector $\theta$ in the relative interior of $R_x$. 
Hence the normal fan $\Fcal_\Xcal$ partitions the parameter space of the independence model into regions of distributions with maxima at different inputs. 

\subsection*{Inference functions and slicings}

For any choice of parameters of the model $\RBM_{\Xcal,\Ycal}$, there is an {\em inference function} $\pi \colon \Xcal\to \Ycal$, (or more generally $\pi\colon\Xcal\to 2^\Ycal$), which computes the most likely hidden state 
%, or set of equally most likely hidden states, 
given a visible state. These functions are not necessarily injective nor surjective. 
For a visible state $x$, the conditional distribution on the hidden states is a product distribution $p(y|X=x) = \frac1Z \exp(\langle \Theta A^{(\Xcal)}_x, A^{(\Ycal)}_y\rangle)$ which is maximized at the state $y$ for which $\Theta A^{(\Xcal)}_x\in R_y$. 
The preimages of the cones $R_y$ by the map $\Theta$ partition the input space $\R^{d_\Xcal}$ and are called {\em inference regions}. See Figure~\ref{slicingsfig} and Example~\ref{ex:maxcones}.

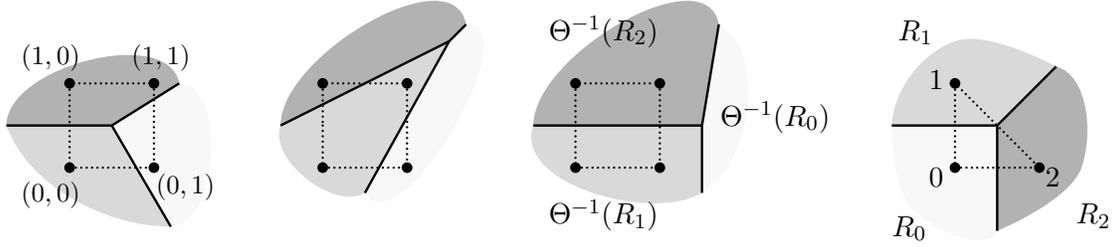
\begin{figure}
\setlength{\unitlength}{1.07\textwidth}
\begin{center}
\begin{picture}(.92,.2)(0,.04)
\put(0,0.02){\includegraphics[trim=4.5cm 19.2cm 4cm 2cm, clip=true, width=.9\unitlength]{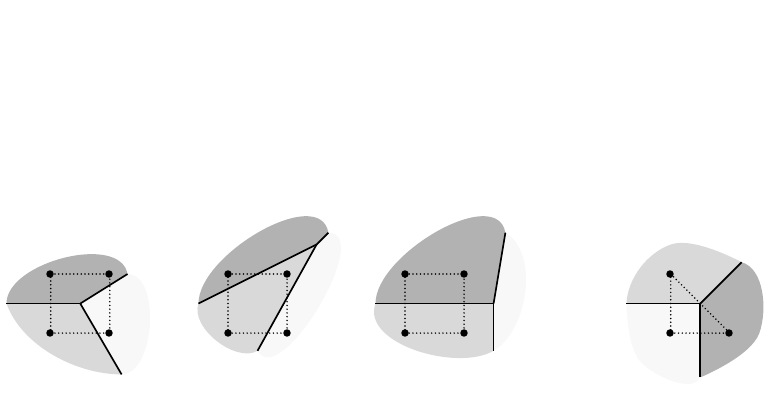}}
\put(.73,.04){$R_0$}
\put(.735,.2){$R_1$}
\put(.88,.05){$R_2$}
\put(.76,.08){$0$}
\put(.76,.16){$1$}
\put(.855,.08){$2$}
\put(0.02,.068){\small$(0,0)$}
\put(0.13,.075){\small$(0,1)$}
\put(0.02,.18){\small$(1,0)$}
\put(0.11,.18){\small$(1,1)$}
\put(.45,.2){$\Theta^{-1}(R_2)$}
\put(.45,.05){$\Theta^{-1}(R_1)$}
\put(.59,.13){$\Theta^{-1}(R_0)$}
\end{picture}
\end{center}
\caption{Three slicings of a square by the normal fan of a triangle with maximal cones $R_0$, $R_1$, and $R_2$, corresponding to three possible inference functions of $\RBM_{\{0,1\}^2,\{0,1,2\}}$. 
%Each vertex of the $2$-cube is a column vector of the sufficient statistics matrix of the $2$-bit independence model. Each vertex of the $2$-simplex is a column vector of the sufficient statistics matrix of the independence model of one single ternary variable (equal to $\Delta(\{0,1,2\})$). 
}
\label{slicingsfig}
\end{figure}

\begin{definition}\label{defslicing}
\begin{normalfont}
A {\em $\Ycal$-slicing} of a finite set $\Zcal\subset\R^{d_\Xcal}$ %$\Zcal_\Xcal=\{A^{(\Xcal)}_x\colon x\in\Xcal\}$ %$\subset\R^d$ 
is a partition of $\Zcal$ into the preimages 
of the cones $R_y$, $y\in\Ycal$ 
by a linear map $\Theta\colon \R^{d_\Xcal}\to\R^{d_\Ycal}$.  We assume  that $\Theta$ is generic, such that it maps each element of $\Zcal$ into the interior of some $R_y$. 
\end{normalfont}
\end{definition}

For example, when $\Ycal=\{0,1\}$, the fan $\Fcal_\Ycal$ consists of a hyperplane and the two closed half-spaces defined by that hyperplane. 
A $\Ycal$-slicing is in this case a standard slicing by a hyperplane. 

\begin{example}\label{ex:maxcones}
\begin{normalfont}
Let $\Xcal=\{0,1,2\}\times\{0,1\}$ and  $\Ycal=\{0,1\}^4$.  
The maximal cones $R_y$, $y\in\Ycal$ of the normal fan of the $4$-cube with vertices $\{0,1\}^4$ are the closed orthants of~$\R^4$.  
The $6$ vertices $\{A^{(\Xcal)}_x\colon x\in\Xcal\}$ of the prism $\Delta(\{0,1,2\})\times\Delta(\{0,1\})$ can be mapped into $6$ distinct orthants of~$\R^4$, each orthant with an even number of positive coordinates: 
\begin{equation}
\renewcommand{\arraystretch}{.8}
\underset{\Theta}{
\left(\begin{array}{r  r  r  r}
\! 3&\! -2&\! -2&\!  -2\!\\
\! 1&\!  2&\! -2&\!  -2\!\\
\! 1&\! -2&\! -2&\!   2\!\\
\! 1&\! -2&\!  2&\!  -2\!
\end{array}
\right)
}
\underset{A^{(\Xcal)}}{
\left(\begin{array}{c c c c c c}
 1 &1 &1 & 1&1 &1 \\
 1 &1 &1 &0 &0 &0 \\
 1 &0 &0 &1 &0 &0 \\ 
 0 &1 &0 &0 &1 &0 
\end{array}
\right)
}
=
\left(\begin{array}{r r r r r r}
\!-1 &\!-1 &\!  1 &\! 1 &\! 1 &  3\!\\
\! 1 &\! 1 &\!  3 &\! -1 &\!-1 &  1\!\\
\!-3 &\! 1 &\! -1 &\!-1 &\! 3 &  1\!\\ 
\! 1 &\!-3 &\! -1 &\! 3 &\!-1 & 1\!
\end{array}
\right) . 
\end{equation}
\end{normalfont}
\end{example}

Even in the case of one single hidden unit the slicings can be complex, but the following simple type of slicing is always available. 

\begin{proposition}\label{paralelslicings}
%If $ \Ycal'\subseteq\Ycal=\{0,1\ldots,k-1\}$, then  every $\Ycal'$-slicing is a $\Ycal$-slicing. 
Any slicing by $k-1$ parallel hyperplanes is a $\{1,2,\ldots,k\}$-slicing. 
\end{proposition}
\begin{proof}\,  
We show that there is a line $\Lcal= \{\lambda r - b \colon \lambda\in\R\}$, $r,b \in \mathbb{R}^k$ intersecting all cells of $\Fcal_\Ycal$, $\Ycal=\{1,\ldots,k\}$. 
We need to show that there is a choice of $r$ and $b$ such that for every $y\in\Ycal$ the set $I_y\subseteq\R$ of all $\lambda$ with $\langle \lambda r - b, (\be_y-\be_z)\rangle>0$ for all $z\in\Ycal\setminus\{y\}$ has a non-empty interior. 
Now, $I_y$ is the set of $\lambda$ with 
\begin{equation}
 \lambda {(r_y- r_z)}{} > b_y-b_z  \quad\text{ for all  $z\neq y$}.
 \label{finidiff}
\end{equation}
Choosing $b_1< \cdots < b_k$ and   $r_y= f(b_y)$, where $f$ is a strictly increasing and strictly concave function, we get $I_1=(-\infty , \frac{b_{2} -b_{1}}{ r_2 - r_1 })$, 
$I_y=( \frac{b_{y } -b_{y-1}}{ r_y - r_{y-1} } ,  \frac{b_{y+1} -b_y }{ r_{y+1} - r_y } )$ for  $y = 2,3,\ldots, k-1 $, and 
 $I_k=(\frac{b_{k } -b_{k-1}}{ r_k - r_{k-1} }, \infty)$. 
The lengths $\infty, l_2, \ldots, l_{k-1}, \infty$ of the intervals $I_1,\ldots,I_k$ can be adjusted arbitrarily  by choosing suitable differences $r_{j+1} - r_{j}$ for all $j=1,\ldots,k-1$. 
\end{proof} 

\subsection*{Strong modes} 

%A strong mode of a distribution $p$ on $\Xcal$ is a point $x\in\Xcal$ such that $p(x)>\sum_{y\in d_{H}(x,y)=1}p(y)$, where $d_H(x,y)$ is the Hamming distance between $x$ and $y$. 
Recall the definition of strong modes given in page~\pageref{strongmodepage}. 
\begin{lemma}\label{strmod1}\label{strongmodesdiscreteRBM}
Let $\Ccal\subseteq\Xcal$ be a set of arrays which are pairwise different in at least two entries (a code of minimum distance two). 

\begin{itemize}
\item 
If $\RBM_{\Xcal,\Ycal}$ contains a probability distribution with strong modes $\Ccal$, then 
there is a linear map $\Theta$ of $\{A^{(\Ycal)}_y\colon y\in\Ycal\}$ into the $\Ccal$-cells of $\Fcal_\Xcal$ 
(the cones $R_x$ above the codewords $x\in\Ccal$)  % in the normal fan) of  
sending at least one vertex into each cell. 
\item 
%On the other hand, 
If there is a linear map $\Theta$ of $\{A^{(\Ycal)}_y\colon y\in\Ycal\}$ %the vertices of %$\times_{j\in[m]}\Delta_{ \Ycal_j }$ $Q_\Ycal$ 
into the $\Ccal$-cells %$R_x$, $x\in\Ccal$ 
of  $\Fcal_\Xcal$, with $\max_x\{ \langle\Theta^\top A^{(\Ycal)}_y,A^{(\Xcal)}_x\rangle\}=c$ for all $y\in\Ycal$, then $\RBM_{\Xcal,\Ycal}$ contains a probability distribution with strong modes $\Ccal$. 
\end{itemize}
\end{lemma}
\begin{proof}\, %[Proof of Lemma~\ref{strmod1}]
This is by Proposition~\ref{discrRBMmixtcomps} and Lemma~\ref{strongmodeslemma}. % and the definition of the normal fan. 
\end{proof}

A simple consequence of the previous lemma is that if the model $\RBM_{\Xcal,\Ycal}$ is a universal approximator of distributions on $\Xcal$, 
then necessarily the number of hidden states is at least as large as the maximum code of visible states of minimum distance two, $|\Ycal|\geq\Amd_\Xcal(2)$. 
Hence discrete RBMs may not be universal approximators even when their parameter count surpasses the dimension of the ambient probability simplex. 

\begin{example}
\begin{normalfont}
Let $\Xcal=\{0,1,2\}^n$ and  $\Ycal=\{0,1, \ldots,4\}^m$. 
In this case $\Amd_\Xcal(2)=3^{n-1}$. 
If $\RBM_{\Xcal,\Ycal}$ is  a universal approximator with  $n=3$ and $n=4$, then $m\geq 2$ and $m\geq 3$, respectively, although the smallest $m$ for which $\RBM_{\Xcal,\Ycal}$ has $3^n -1$ %=\dim(\Delta(\Xcal))$ 
parameters is $m=1$ and $m=2$, respectively. 
\end{normalfont}
\end{example}

Using Lemma~\ref{strmod1} and the analysis of \citep{MontufarMorton2012} gives the following.  
\begin{proposition}\label{strongmodepro}
If $4\lceil m/3 \rceil\leq n$, then $\RBM_{\Xcal,\Ycal}$ contains distributions with $2^m$ strong modes. 
\end{proposition}

\section{Representational power and approximation errors}\label{section:expressive}

In this section we describe submodels of discrete RBMs and use them to provide bounds on the model approximation errors depending on the number of units and their state spaces. 
Universal approximation results follow as special cases with vanishing approximation error. 

\begin{theorem}\label{corpropouniv}
%Let $d_\Ycal=1+\sum_{j=1}^m (|\Ycal_j|-1)$. 
The  model  $\RBM_{\Xcal,\Ycal}$ can approximate the following arbitrarily well: 
\begin{itemize}
\item 
Any mixture of $d_\Ycal=1+\sum_{j=1}^m (|\Ycal_j|-1)$ product distributions with disjoint supports. 
\item 
When $ d_\Ycal \geq (\prod_{i\in [k]}|\Xcal_i| ) / \max_{j\in[k]}|\Xcal_j|$ for some $k\leq n$, 
any distribution from the model $\Pcal$ of distributions with constant value on each block  $\{x_1\}\times\cdots\times\{x_k\}\times\Xcal_{k+1}\times\cdots\times\Xcal_n$ for all $x_i\in\Xcal_{i}$, for all $i\in[k]$. 
\item 
Any probability distribution with support contained in the union of $d_\Ycal$ sets of the form $\{x_1\}\times\cdots\times\{x_{k-1}\}\times\Xcal_k\times\{x_{k+1}\}\times\cdots\times\{x_n\}$.  
\end{itemize}
\end{theorem}
\begin{proof}\,  
By Proposition~\ref{prodmixt} the model $\RBM_{\Xcal,\Ycal}$ contains any Hadamard product $p^{(1)}\circ\cdots\circ p^{(m)}$ with mixtures of products as factors, $p^{(j)} \in \Mcal_{\Xcal,|\Ycal_j|}$ for all $j\in[m]$. In particular, it contains $p=p^{(0)}\circ(\mathds{1} +\tilde \lambda_1 \tilde p^{(1)} )\circ\cdots\circ(\mathds{1} +\tilde \lambda_m \tilde p^{(m)})$, where $p^{(0)}\in\Ecal_\Xcal$,  $\tilde p^{(j)}\in\Mcal_{\Xcal,|\Ycal_j|-1}$, and $\tilde \lambda_j\in\R_+$. 
Choosing the factors $\tilde p^{(j)}$ with pairwise disjoint supports shows that  $p=\sum_{j=0}^m\lambda_j p^{(j)}$, whereby $p^{(0)}$ can be any product distribution and $p^{(j)}$ can be any distribution from $\Mcal_{\Xcal,|\Ycal_j|-1}$ for all $j\in[m]$, as long as $\supp(p^{(j)})\cap\supp(p^{(j')})$ for all $j\neq j'$. This proves the first item.  

For the second item: 
Any point in the set $\Pcal$ is a mixture of uniform distributions supported on the disjoint blocks $\{x_1\}\times\cdots\times\{x_k\}\times\Xcal_{k+1}\times\cdots\times\Xcal_n$ for all $(x_1,\ldots, x_k)\in\Xcal_1\times\cdots\times\Xcal_k$. 
%These are product distributions with disjoint supports. 
Each of these uniform distributions is a product distribution, since it factorizes as 
$p_{x_1,\ldots,x_k}=\prod_{i\in[k]}\delta_{x_i} \prod_{i\in[n]\setminus[k]}u_i$, where $u_i$ denotes the uniform distribution on $\Xcal_i$. 
For any  $j\in[k]$ any mixture $\sum_{x_j\in\Xcal_j} \lambda_{x_j} p_{x_1,\ldots,x_k}$ is also a product distribution, since it factorizes as
\begin{equation}
\Big(\sum_{x_j\in\Xcal_j} \lambda_{x_j}\delta_{x_j}\Big) \prod_{i\in[k]\setminus\{j\}}\delta_{x_i}
 \prod_{i\in[n]\setminus[k]}u_i . \label{mixtprod}
 \end{equation}
Hence  any distribution from the set $\Pcal$ is a mixture of $(\prod_{i\in[k]}|\Xcal_i|)/\max_{j\in[k]}|\Xcal_j|$ product distributions with disjoint supports. The claim now follows from the first item. 
% of the form given in eq.~\eqref{mixtprod}. 

For the third item: 
The model $\overline{\Ecal_\Xcal}$ contains any distribution with support of the form $\{x_1\}\times\cdots\times\{x_{k-1}\}\times\Xcal_k\times\{x_{k+1}\}\times\cdots\times\{x_n\}$. 
Hence, by the first item, the RBM model can approximate any distribution arbitrarily well whose support can be covered by $d_\Ycal$ sets of that form. 
\end{proof}

We now analyse the RBM model approximation errors.  
Let $p$ and $q$ be two probability distributions on $\Xcal$. 
The Kullback-Leibler divergence from $p$ to $q$ is defined as $D(p\|q):=\sum_{x\in\Xcal} p(x) \log \frac{p(x)}{q(x)}$ when $\supp(p)\subseteq\supp(q)$ and $D(p\|q):=\infty$ otherwise. 
The divergence from $p$ to a model $\Mcal\subseteq\Delta(\Xcal)$ is defined as $D(p\|\Mcal):=\inf_{q\in \Mcal}D(p\| q)$ and the maximal approximation error of $\Mcal$ is  $\sup_{p\in\Delta(\Xcal)}D(p\|\Mcal)$. 

The maximal approximation error of the independence model $\Ecal_\Xcal$ satisfies $\sup_{p\in\Delta(\Xcal)} D(p\| \Ecal_\Xcal) \leq |\Xcal|/ \max_{i\in[n]}|\Xcal_i|$, with equality when all units have the same number of states~\citep[see][Corollary~4.10]{AyKnauf06:Maximizing_Multiinformation}. 

\begin{theorem}\label{approxerrordiscreteRBM}
If $\prod_{i\in[n]\setminus\Lambda}|\Xcal_{i}| \leq 1+\sum_{j\in[m]}(|\Ycal_j|-1)=d_\Ycal$ for some $\Lambda\subseteq [n]$, then the Kullback-Leibler divergence from any distribution $p$ on $\Xcal$
 to the model $\RBM_{\Xcal,\Ycal}$ is bounded by 
\begin{equation*}
 D(p\|\RBM_{\Xcal,\Ycal}) \leq  \log \frac{\prod_{i\in\Lambda}|\Xcal_{i}|}{\max_{i\in\Lambda}{|\Xcal_i|}}  .
\end{equation*}
In particular, \label{universaldiscreteRBM}
the model $\RBM_{\Xcal,\Ycal}$ is a universal approximator 
%of distributions on $\Xcal$ 
whenever $d_\Ycal \geq {|\Xcal|}/{\max_{i\in[n]}|\Xcal_i|}$.  
\end{theorem}
\begin{proof}\,  
The submodel $\Pcal$ of $\RBM_{\Xcal,\Ycal}$ described in the second item of Theorem~\ref{corpropouniv} is a {\em partition model}. 
The maximal divergence from such a model is equal to the logarithm of the cardinality of the largest block with constant values~\citep[see][]{MatusAy03:On_Maximization_of_the_Information_Divergence}. % See~\citep{NIPS2011_0307}. 
Thus $\max_p D(p\|\RBM_{\Xcal,\Ycal})\leq \max_p D(p\|\Pcal) = \log \left( ( \prod_{i\in\Lambda}|\Xcal_{i}|) / \max_{i\in\Lambda}{|\Xcal_i|} \right)$, as was claimed. 
\end{proof}

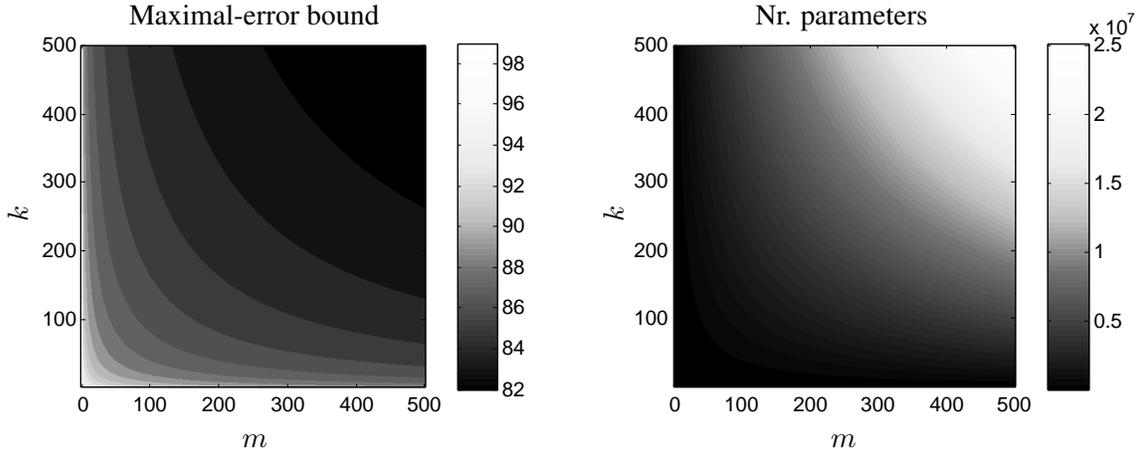
\begin{figure}
\centering
\setlength{\unitlength}{15cm}
\begin{picture}(1.2,.395)
\put(0.04,0){\includegraphics[clip=true, trim= 1.5cm 10.1cm 1cm 10cm, scale=.77]{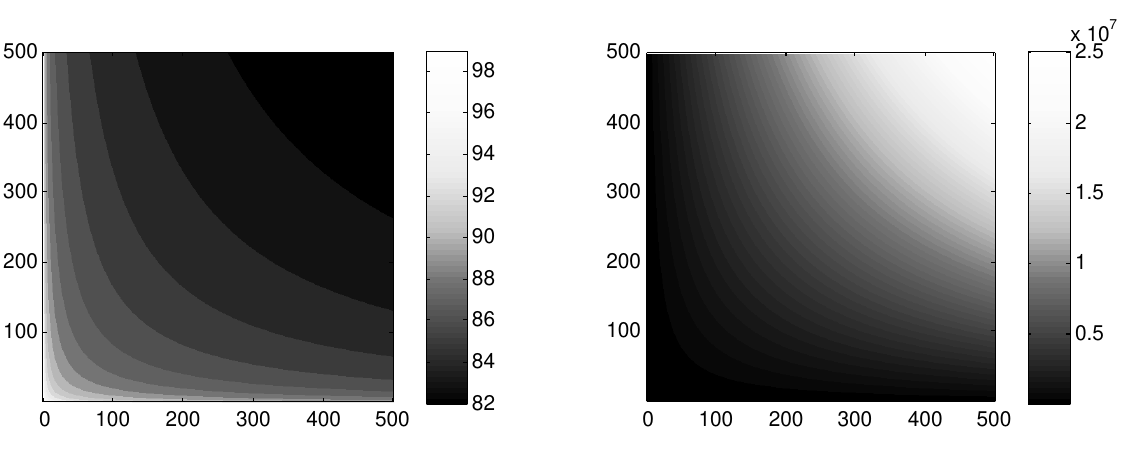}} 
\put(0.03,.2){\begin{rotate}{90}$k$\end{rotate}}
\put(.22,0){$m$}
\put(0.12,.375){Maximal-error bound}
\put(.558,.2){\begin{rotate}{90}$k$\end{rotate}}
\put(.742,0){$m$}
\put(0.675,.375){Nr. parameters}
\end{picture}
\caption{Illustration of Theorem~\ref{approxerrordiscreteRBM}. 
The left panel shows a heat map of the upper bound on the Kullback-Leibler approximation errors of discrete RBMs with 100 visible binary units and the right panel shows a map of the total number of model parameters, both depending on the number of hidden units $m$ and their possible states $k=|\Ycal_j|$ for all $j\in[m]$. 
}\label{errorbounds}
\end{figure}
Theorem~\ref{approxerrordiscreteRBM} shows that, on a large scale, the maximal model approximation error of $\RBM_{\Xcal,\Ycal}$ is smaller than that of the independence model $\Ecal_\Xcal$ by at least $\log(1+\sum_{j\in[m]}(|\Ycal_j|-1))$, or vanishes. 
The theorem is illustrated in Figure~\ref{errorbounds}. 
The line $k=2$ shows bounds on the approximation error of binary RBMs with $m$ hidden units, previously treated in~\citep[][Theorem~5.1]{NIPS2011_0307}, and the line $m=1$ shows bounds for na\"ive Bayes models with $k$ hidden classes.

\section{Dimension}\label{section:algebraic}

In this section we study the dimension of the model $\RBM_{\Xcal,\Ycal}$. 
One reason RBMs are attractive is that they have a large learning capacity, e.g.\ may be built with millions of parameters.  Dimension calculations show whether those parameters are wasted, or translate into higher-dimensional spaces of representable distributions. 
Our analysis builds on previous work by~\citet*{Cueto2010}, where binary RBMs are treated.  
The idea is to bound the dimension from below by the dimension of a related max-plus model, called the tropical RBM model~\citep{tropical}, 
and from above by the dimension expected from counting parameters.

The dimension of a discrete RBM model can be bounded from above not only by its expected dimension, but also by a function of the dimension of its Hadamard factors:  
\begin{proposition}\label{remdimu}
The dimension of $\RBM_{\Xcal,\Ycal}$ is bounded as 
\begin{equation}
\dim(\RBM_{\Xcal,\Ycal})\leq \dim(\Mcal_{\Xcal,|\Ycal_i|}) + \sum_{j\in[m]\setminus\{i\}}\dim(\Mcal_{\Xcal, |\Ycal_j|-1})+(m-1) \quad\text{for all $i\in[m]$} .   
\label{remdimueq}
\end{equation}
\end{proposition}
\begin{proof}\,  
Let $u$ denote the uniform distribution. 
Note that $\Ecal_\Xcal \circ\Ecal_\Xcal= \Ecal_\Xcal$ and also $\Ecal_\Xcal\circ \Mcal_{\Xcal,k}=\Mcal_{\Xcal,k}$. 
This observation, together with Proposition~\ref{prodmixt}, shows that the RBM model can be factorized as 
\begin{equation*}
\RBM_{\Xcal,\Ycal} = (\Mcal_{\Xcal,|\Ycal_1|})\circ (\lambda_1 u+(1- \lambda_1) \Mcal_{\Xcal,|\Ycal_1|})\circ \cdots\circ(\lambda_m u+ (1-\lambda_m)\Mcal_{\Xcal,|\Ycal_m|-1}) , 
\end{equation*}
from which the claim follows. 
\end{proof}

By the previous proposition, the model $\RBM_{\Xcal,\Ycal}$ can have the expected dimension only if (i) the right hand side of eq.~\eqref{remdimueq} equals $|\Xcal|-1$, or (ii) each mixture model $\Mcal_{\Xcal, k}$ has the expected dimension %, equal to $k\dim(\Ecal_\Xcal)+k-1$,  
for all $k\leq \max_{j\in[m]}|\Ycal_j|$. 
Sometimes none of both conditions is satisfied and the models `waste' parameters: 

\begin{example} 
\begin{normalfont}
The $k$-mixture of the independence model on $\Xcal_1\times\Xcal_2$ is a subset of the set of $|\Xcal_1|\times|\Xcal_2|$ matrices with non-negative entries and rank at most $k$. 
It is known that the set of $M\times N$ matrices of rank at most $k$ has dimension $k(M+N-k)$ for all $1\leq k<\min\{M,N\}$. 
Hence 
%$\dim( \Mixt^k(\Ecal_{\Xcal_1\times\Xcal_2})) = k(|\Xcal_1|+|\Xcal_2|-k) -1 < k(|\Xcal_1|+|\Xcal_2|-1)-1$ 
the model $\Mcal_{\Xcal_1\times\Xcal_2,k}$ has dimension smaller than its parameter count whenever $1< k < \min\{|\Xcal_1|,|\Xcal_2|\}$. 
By Proposition~\ref{remdimu} if $(\sum_{j\in[m]}(|\Ycal_j|-1)+1)(|\Xcal_1|+|\Xcal_2|-1) \leq |\Xcal_1\times\Xcal_2|$ and $1<|\Ycal_j|\leq \min\{|\Xcal_1|,|\Xcal_2|\}$ for some $j\in[m]$, then $\RBM_{\Xcal_1\times\Xcal_2,\Ycal}$ does not have the expected dimension. 
\end{normalfont}
\end{example}

The next theorem indicates choices of $\Xcal$ and $\Ycal$ for which the model $\RBM_{\Xcal,\Ycal}$ has the expected dimension. 
Given a sufficient statistics matrix $A^{(\Xcal)}$, we say that a set $\Zcal\subseteq\Xcal$ has full rank when the matrix with columns $\{A^{(\Xcal)}_x\colon x\in\Zcal\}$ has full rank.  
\begin{theorem}\label{theorem:dimension}
When $\Xcal$ contains $m$ disjoint Hamming balls of radii \mbox{$2(|\Ycal_j|-1)-1$}, $j\in[m]$ 
and the subset of $\Xcal$ not intersected by these balls has full rank, then the model $\RBM_{\Xcal,\Ycal}$ has dimension equal to the number of model parameters,  $$\dim(\RBM_{\Xcal,\Ycal}) = (1+\sum_{i\in[n]}(|\Xcal_i|-1))(1+\sum_{j\in[m]}(|\Ycal_j|-1))-1. $$  
On the other hand, if $m$ Hamming balls of radius one cover $\Xcal$, then   $$\dim(\RBM_{\Xcal,\Ycal})=|\Xcal|-1. $$  
\end{theorem}

In order to prove this theorem we will need two main tools: 
slicings by normal fans of simplices, described in Section~\ref{section:inference}, and the tropical RBM model, described in Section~\ref{section:tropical model}. 
The theorem will follow from the analysis contained in Section~\ref{section:tropical model}. 

\section{Tropical model}\label{section:tropical model}

\begin{definition}
\begin{normalfont}
The tropical model $\RBM_{\Xcal,\Ycal}^{\text{\normalfont tropical}}$ is the image of the tropical morphism 
\begin{equation}
\R^{d_\Xcal  d_\Ycal}\ni\theta\;\;\mapsto\;\; \Phi(v;\theta)= %q_\theta(v) = 
\max\{\langle\theta, A^{(\Xcal,\Ycal)}_{(v,h)}\rangle\colon h\in\Ycal\} \quad \text{ for all }v\in\Xcal, %\theta\in\R^d ,
\end{equation}
which evaluates  
$\log( \frac{1}{Z(\theta)} \sum_{h\in\Ycal} \exp(\langle \theta,A^{(\Xcal,\Ycal)}_{(v,h)}\rangle))$
for all $v\in\Xcal$ for each  $\theta$  
within the max-plus algebra (addition becomes $a+b=\max\{a,b\}$) up to additive constants independent of $v$ (i.e., disregarding the normalization factor $Z(\theta)$).  
\end{normalfont}
\end{definition}
The idea behind this definition is that $\log(\exp(a)+\exp(b))\approx\max\{a,b\}$ when $a$ and $b$ have different order of magnitude. 
The tropical model captures important properties of the original model. Of particular interest is following consequence of the Bieri-Groves theorem~\citep[see][]{Draisma}, which gives us a tool to estimate the dimension of $\RBM_{\Xcal,\Ycal}$: 
\begin{gather}
\dim(\TRBM_{\Xcal,\Ycal})\leq \dim(\RBM_{\Xcal,\Ycal}) \leq\min\{\dim( \Ecal_{\Xcal,\Ycal}),|\Xcal|-1\} .  \label{dimbounds}
\end{gather}  

The following Theorem~\ref{tropicalRBM} describes the regions of linearity of the map $\Phi$. 
Each of these regions corresponds to a collection of $\Ycal_j$-slicings (see Definition~\ref{defslicing}) of the set $\{ A^{(\Xcal)}_x\colon x\in\Xcal\}$ 
for all $j\in[m]$. 
This result allows us to express the dimension of $\TRBM_{\Xcal,\Ycal}$ as the maximum rank of a class of matrices defined by collections of slicings.

For each $j\in[m]$ let $C_j=\{C_{j,1},\ldots,C_{j, |\Ycal_j|}\}$ be a $\Ycal_j$-slicing of $\{A^{(\Xcal)}_x\colon x\in\Xcal\}$ and let $A_{C_{j,k}}$ be the $|\Xcal|\times d_\Xcal$-matrix with $x$-th row equal to $( A^{(\Xcal)}_x)^\top$ when $x \in C_{j,k}$ and equal to a row of zeros otherwise. Let $A_{C_j}=(A_{C_{j,1}}|\cdots|A_{C_{j,|\Ycal_j|}}) \in \R^{|\Xcal|\times |\Ycal_j|d_\Xcal}$ and $d =  \sum_{j\in[m]} |\Ycal_j|d_\Xcal$. 

\begin{theorem}\label{tropicalRBM} 
On each region of linearity, %corresponding to a collection of $m$ $\Ycal_j$-slicings, 
the tropical morphism $\Phi$ is the linear map $\R^{d}\to \TRBM_{\Xcal,\Ycal}$ represented by the $|\Xcal|\times d$-matrix 
\begin{equation*}
\Acal = (A_{C_1}|\cdots|A_{C_m}),
\end{equation*}
modulo constant functions. 
In particular, $\dim(\TRBM_{\Xcal,\Ycal})+1$ is the maximum rank of $\Acal$ over all possible collections of slicings $C_1, \ldots, C_m$. 
\end{theorem}
\begin{proof}\,  
Again use the homogeneous version of the matrix $A^{(\Xcal,\Ycal)}$ as in the proof of Proposition \ref{prodmixt}; this will not affect the rank of $\mathcal{A}$.  
Let $\theta_{h_j}=(\theta_{\{j,i\},(h_j,x_i)})_{i\in[n], x_i\in\Xcal_i}$ and let $A_{h_j}$ denote the submatrix of $A^{(\Xcal,\Ycal)}$ containing the rows with indices $\{\{j,i\},(h_j,x_i)\colon {i\in[n], x_i\in\Xcal_i}\}$. 
For any given $v\in\Xcal$ we have 
\begin{equation*}
\max\Big\{\big\langle \theta, A^{(\Xcal,\Ycal)}_{(v,h)}\big\rangle \colon h\in\Ycal\Big\} 
=\sum_{j\in[m]} \max\Big\{\big\langle \theta_{h_j}, A_{h_j} (v,h_j)\big\rangle\colon h_j\in\Ycal_j\Big\}, 
\end{equation*}
from which the claim follows.
\end{proof}

In the following we evaluate the maximum rank of the matrix $\Acal$ 
for various choices of $\Xcal$ and $\Ycal$ by examining good slicings. 
We focus on slicings by parallel hyperplanes. 

\begin{lemma} \label{lem:lemma25}
For any $x^\ast\in \Xcal$ and $0<k<n$ the affine hull of the set $\{A_x^{(\Xcal)}\colon d_H(x,x^\ast)=k\}$ has dimension $\sum_{i\in[n]}(|\Xcal_i|-1)-1$. 
\end{lemma}
\begin{proof}\,  
Without loss of generality let $x^\ast=(0,\ldots, 0)$. 
The set $\Zcal^k:=\{A_x^{(\Xcal)}\colon d_H(x,x^\ast)=k\}$ is the intersection of $\{A_x^{(\Xcal)}\colon x\in\Xcal \}$ with the hyperplane $H^k:=\{z \colon \langle \mathds{1}, z\rangle=k+1\}$. 
Now note that the two vertices of an edge of $Q_\Xcal$ either lie in the same hyperplane $H^l$, or in two adjacent parallel hyperplanes $H^l$ and $H^{l+1}$, with $l\in\mathbb{N}$. 
Hence the hyperplane $H^k$ does not slice any edges of $Q_\Xcal$ and $\operatorname{conv}(\Zcal^k)=Q_\Xcal\cap H^k$.  
The set $\Zcal^k$ is not contained in any proper face of $Q_\Xcal$ and hence $\operatorname{conv}(\Zcal^k)$ intersects the interior of $Q_\Xcal$. Thus $\dim(\operatorname{conv}(\Zcal^k)) = \dim(Q_\Xcal)-1$, as was claimed. 
\end{proof}

Lemma \ref{lem:lemma25} implies the following. 
\begin{corollary}\label{lem:lemma24}
Let  $x\in\Xcal$, and  $2k-3\leq n$. 
There is a slicing $C_1=\{C_{1,1},\ldots,C_{1,k}\}$ of $\Xcal$ by $k-1$ parallel hyperplanes such that $\cup_{l=1}^{k-1}C_{1,l}=B_x(2k-3)$ is the Hamming ball of radius $2k-3$ centered at $x$ and the matrix $A_{C_1}=(A_{C_{1,1}}|\cdots|A_{C_{1,k-1}})$ has full rank. 
\end{corollary}

Recall that $\Amd_\Xcal(d)$ denotes the maximal cardinality of a subset of $\Xcal$ of minimum Hamming distance at least $d$. 
When $\Xcal=\{0,1,\ldots, q-1\}^n$ we write  $\Amd_q(n,d)$. 
Let $\Kmd_\Xcal(d)$ denote the minimal cardinality of a subset of $\Xcal$ with covering radius $d$.

\begin{proposition}[Binary visible units]\label{corbinhi}
Let $\Xcal=\{0,1\}^n$ and $|\Ycal_j|=s_j$ for all $j\in[m]$. 
If $\Xcal$ contains $m$ disjoint Hamming balls of radii $2s_j-3$, $j\in[m]$ whose complement has full rank, 
then $\RBM_{\Xcal,\Ycal}^{\text{\normalfont tropical}}$ has the expected dimension, $\min\{\sum_{j\in[m]} (s_j-1)(n+1) +n , 2^n-1\}$. 
\end{proposition}
In particular, if $\Xcal=\{0,1\}^n$ and $\Ycal=\{0,1,\ldots,s-1\}^m$ with  $m < \Amd_2(n,d)$ and $d=4(s-1)-1$, then $\RBM_{\Xcal,\Ycal}$ has the expected dimension. 
It is known that $\Amd_2(n,d)\geq 2^{n-\lceil \log_2(\sum_{j=0}^{d-2} {n-1\choose j})\rceil}$.

\begin{proposition}[Binary hidden units]\label{binhidcor}
Let $\Ycal = \{0,1\}^m$ and $\Xcal$ be arbitrary.  
\begin{itemize}
\item If $m+1\leq\Amd_\Xcal(3)$, then $\RBM_{\Xcal,\{0,1\}^m}^{\text{\normalfont tropical}}$ has dimension $(1+m)(1+\sum_{i\in[n]}(|\Xcal_i|-1)) -1$. 
\item If $m+1\geq\Kmd_\Xcal(1)$, then $\RBM_{\Xcal,\{0,1\}^m}^{\text{\normalfont tropical}}$ has dimension $|\Xcal|-1$. 
\end{itemize}

\noindent
Let $\Ycal=\{0,1\}^m$ and $\Xcal=\{0,1,\ldots,q-1\}^n$, where $q$ is a prime power.  
\begin{itemize}
\item 
If $m+1\leq q^{n-\left\lceil\log_q (1+ (n-1)(q-1) +1)\right\rceil}$, then $\RBM_{\Xcal,\Ycal}^{\text{\normalfont tropical}}$ has dimension \newline $(1+m)(1+\sum_{i\in[n]}(|\Xcal_i|-1)) -1$. 

\item  
If $n=(q^r-1)/(q-1)$ for some $r\geq2$, then $\Acal_\Xcal(3)=\Kmd_\Xcal(1)$, and $\RBM_{\Xcal,\Ycal}^{\text{\normalfont tropical}}$ has the expected dimension for any $m$. 
\end{itemize}
\end{proposition}
In particular, when all units are binary and $m< 2^{n-\lceil\log_2(n+1) \rceil}$, then $\RBM_{\Xcal,\Ycal}$ has the expected dimension; this was shown in~\citep{Cueto2010}.

\begin{proposition}[Arbitrary sized units]\label{generalcasecor}
If $\Xcal$ contains $m$ disjoint Hamming balls of radii $ 2|\Ycal_1|-3,\ldots, 2|\Ycal_m|-3$, and the complement of their union has full rank, then 
$\RBM^{\text{\normalfont tropical}}_{\Xcal,\Ycal}$ has the expected dimension. 
\end{proposition}

\begin{proof}\,  
Propositions \ref{corbinhi}, \ref{binhidcor}, and \ref{generalcasecor} follow from Theorem \ref{tropicalRBM} and Corollary \ref{lem:lemma24} together with the following explicit bounds on $\Amd$ 
by~\citep[][]{Gilbert:1952,Varshamov:1957}:  
$$\Amd_q(n,d)\geq \frac{q^n}{\sum_{j=0}^{d-1}{n\choose j} (q-1)^j}.$$ If $q$ is a prime power, then $\Amd_q(n,d)\geq  q^k$, where $k$ is the largest integer with $q^k<\frac{q^n}{\sum_{j=0}^{d-2}{n-1\choose j}(q-1)^j}$. 
In particular, 
$\Amd_2(n,3)\geq 2^k$, where $k$ is the largest integer with $2^k<\frac{2^n}{(n-1)+1}=2^{n-\log_2(n)}$, i.e., $k=n-\lceil \log_2(n+1)\rceil$. 
\end{proof}

\begin{example}
\begin{normalfont}
Many results in coding theory can now be translated directly to statements about the dimension of discrete RBMs.  
Here is an example.  
Let $\Xcal=\{1,2,\ldots,s\}\times\{1,2,\ldots,s\}\times\{1,2,\ldots,t\}$, $s\leq t$. 
The minimum cardinality of a code $C\subseteq\Xcal$ with covering-radius one equals $\Kmd_\Xcal(1) = s^2-\left\lfloor \frac{(3 s -t)^2}{8}\right\rfloor$ if $t\leq 3s$, and $\Kmd_\Xcal(1) = s^2$ otherwise~\citep[see][Theorem~3.7.4]{cohen2005covering}. 
Hence $\TRBM_{\Xcal,\{0,1\}^m}$ has dimension $|\Xcal|-1$ when $m+1\geq s^2-\left\lfloor \frac{(3 s -t)^2}{8}\right\rfloor$ and $t\leq 3 s$, and when $m+1\geq s^2$ and $t>3s$. 
\end{normalfont}
\end{example}

\section{Discussion}\label{section:discussion}

In this note we study the representational power of RBMs with discrete units. 
Our results generalize a diversity of previously known results for standard binary RBMs and na\"ive Bayes models.  They help contrasting the geometric-combinatorial properties of distributed products of experts versus non-distributed mixtures of experts. 

We estimate the number of hidden units for which discrete RBM models can approximate any distribution to any desired accuracy, depending on the cardinalities of their units' state spaces. 
This analysis shows that the maximal approximation error increases at most logarithmically with the total number of visible states and decreases at least logarithmically with the sum of the number of states of the hidden units. This observation could be helpful, for example, in designing a penalty term to allow comparison of models with differing numbers of units. 
It is worth mentioning that the submodels of discrete RBMs described in Theorem~\ref{corpropouniv} can be used not only to estimate the maximal model approximation errors, but also the expected model approximation errors given a prior of target distributions on the probability simplex. 
See~\citep{MonRauh2012} for an exact analysis of Dirichlet priors. 
In future work it would be interesting to study the statistical approximation errors of discrete RBMs and to complement the theory by an empirical evaluation.

The combinatorics of tropical discrete RBMs allows us to relate the dimension of discrete RBM models to the solutions of linear optimization problems and slicings of convex support polytopes by normal fans of simplices. 
We use this to show that the model $\RBM_{\Xcal,\Ycal}$ has the expected dimension for many choices of $\Xcal$ and $\Ycal$, but not for all choices.  
We based our explicit computations of the dimension of RBMs on slicings by collections of parallel hyperplanes, but more general classes of slicings may be considered. 
The same tools presented in this paper can be used to estimate the dimension of a general class of models involving interactions within layers, defined as Kronecker products of hierarchical models~\citep[see][]{MontufarMorton2013b}. 
We think that the geometric-combinatorial picture of discrete RBMs developed in this paper may be helpful in solving various long standing theoretical problems in the future, for example: 
What is the exact dimension of na\"ive Bayes models with general discrete variables? 
What is the smallest number of hidden variables that make an RBM a universal approximator? 
Do binary RBMs always have the expected dimension?

\subsubsection*{Acknowledgments}
We are grateful to the ICLR 2013 community for very valuable comments. 
This work was accomplished in part at the Max Planck Institute for Mathematics in the Sciences. 
%, Leipzig, Germany. % during G.~M.'s visit in September and October 2012. 
This work is supported in part by DARPA grant FA8650-11-1-7145.

\bibliographystyle{abbrvnat}
\bibliography{referenzen}{}

\begin{thebibliography}{39}
\providecommand{\natexlab}[1]{#1}
\providecommand{\url}[1]{\texttt{#1}}
\expandafter\ifx\csname urlstyle\endcsname\relax
  \providecommand{\doi}[1]{doi: #1}\else
  \providecommand{\doi}{doi: \begingroup \urlstyle{rm}\Url}\fi

\bibitem[Aoyagi(2010)]{aoyagi:2010}
M.~Aoyagi.
\newblock Stochastic complexity and generalization error of a {R}estricted
  {B}oltzmann {M}achine in {B}ayesian estimation.
\newblock \emph{J. Mach. Learn. Res.}, 99:\penalty0 1243--1272, August 2010.

\bibitem[Ay and Knauf(2006)]{AyKnauf06:Maximizing_Multiinformation}
N.~Ay and A.~Knauf.
\newblock Maximizing multi-information.
\newblock \emph{Kybernetika}, 42\penalty0 (5):\penalty0 517--538, 2006.

\bibitem[Bengio(2009)]{Bengio-2009}
Y.~Bengio.
\newblock Learning deep architectures for {AI}.
\newblock \emph{Found. Trends Mach. Learn.}, 2\penalty0 (1):\penalty0 1--127,
  2009.

\bibitem[Bengio et~al.(2007)Bengio, Lamblin, Popovici, and
  Larochelle]{Bengio2007}
Y.~Bengio, P.~Lamblin, D.~Popovici, and H.~Larochelle.
\newblock Greedy layer-wise training of deep networks.
\newblock In B.~Sch\"{o}lkopf, J.~Platt, and T.~Hoffman, editors,
  \emph{Advances in Neural Information Processing Systems 19}, pages 153--160.
  MIT Press, Cambridge, MA, 2007.

\bibitem[Carreira-Perpi\~nan and Hinton(2005)]{Carreira2005}
M.~A. Carreira-Perpi\~nan and G.~E. Hinton.
\newblock On contrastive divergence learning.
\newblock In \emph{Proceedings of the 10-th Interantional Workshop on
  Artificial Intelligence and Statistics}, 2005.

\bibitem[Catalisano et~al.(2011)Catalisano, Geramita, and
  Gimigliano]{Catalisano2011}
M.~V. Catalisano, A.~V. Geramita, and A.~Gimigliano.
\newblock Secant varieties of $\mathbb{P}^1\times\dots\times\mathbb{P}^1$
  ($n$-times) are not defective for $n\geq5$.
\newblock \emph{J. Algebraic Geometry}, 20:\penalty0 295--327, 2011.

\bibitem[Cohen et~al.(2005)Cohen, Honkala, Litsyn, and
  Lobstein]{cohen2005covering}
G.~Cohen, I.~Honkala, S.~Litsyn, and A.~Lobstein.
\newblock \emph{Covering Codes}.
\newblock North-Holland Mathematical Library. Elsevier Science, 2005.

\bibitem[Cueto et~al.(2010)Cueto, Morton, and Sturmfels]{Cueto2010}
M.~A. Cueto, J.~Morton, and B.~Sturmfels.
\newblock Geometry of the restricted {B}oltzmann machine.
\newblock In M.~Viana and H.~Wynn, editors, \emph{Algebraic methods in
  statistics and probability II, AMS Special Session}, volume~2. American
  Mathematical Society, 2010.

\bibitem[Dahl et~al.(2012)Dahl, Adams, and Larochelle]{wordrbm}
G.~E. Dahl, R.~P. Adams, and H.~Larochelle.
\newblock Training restricted {B}oltzmann machines on word observations.
\newblock \emph{arXiv:1202.5695}, 2012.

\bibitem[Dempster et~al.(1977)Dempster, Laird, and Rubin]{dempsterEM}
A.~P. Dempster, N.~M. Laird, and D.~B. Rubin.
\newblock Maximum likelihood from incomplete data via the {EM} algorithm.
\newblock \emph{Journal of the Royal Statistical Society. Series B
  (Methodological)}, 39\penalty0 (1):\penalty0 1--38, 1977.

\bibitem[Draisma(2008)]{Draisma}
J.~Draisma.
\newblock A tropical approach to secant dimensions.
\newblock \emph{J. Pure Appl. Algebra}, 212\penalty0 (2):\penalty0 349--363,
  2008.

\bibitem[Freund and Haussler(1991)]{Freund1992}
Y.~Freund and D.~Haussler.
\newblock Unsupervised learning of distributions of binary vectors using
  2-layer networks.
\newblock In J.~E. Moody, S.~J. Hanson, and R.~Lippmann, editors,
  \emph{Advances in Neural Information Processing Systems 4}, pages 912--919.
  Morgan Kaufmann, 1991.

\bibitem[Gilbert(1952)]{Gilbert:1952}
E.~N. Gilbert.
\newblock A comparison of signalling alphabets.
\newblock \emph{Bell System Technical Journal}, 31:\penalty0 504--522, 1952.

\bibitem[Hinton(1999)]{Hinton99productsof}
G.~E. Hinton.
\newblock Products of experts.
\newblock In \emph{Proceedings 9-th ICANN}, volume~1, pages 1--6, 1999.

\bibitem[Hinton(2002)]{Hinton2002}
G.~E. Hinton.
\newblock Training products of experts by minimizing contrastive divergence.
\newblock \emph{Neural Computation}, 14:\penalty0 1771--1800, 2002.

\bibitem[Hinton(2010)]{Hinton2010}
G.~E. Hinton.
\newblock A practical guide to training restricted {B}oltzmann machines,
  version 1.
\newblock Technical report, UTML2010-003, University of Toronto, 2010.

\bibitem[Hinton et~al.(2006)Hinton, Osindero, and Teh]{Hinton2006}
G.~E. Hinton, S.~Osindero, and Y.~Teh.
\newblock A fast learning algorithm for deep belief nets.
\newblock \emph{Neural Computation}, 18:\penalty0 1527--1554, 2006.

\bibitem[Le~Roux and Bengio(2008)]{LeRoux2008}
N.~Le~Roux and Y.~Bengio.
\newblock Representational power of restricted {B}oltzmann machines and deep
  belief networks.
\newblock \emph{Neural Computation}, 20\penalty0 (6):\penalty0 1631--1649,
  2008.

\bibitem[Long and Servedio(2010)]{LongServedio10}
P.~M. Long and R.~A. Servedio.
\newblock Restricted {B}oltzmann machines are hard to approximately evaluate or
  simulate.
\newblock In J.~F{\"u}rnkranz and T.~Joachims, editors, \emph{Proceedings of
  the 27th International Conference on Machine Learning (ICML-10)}, pages
  703--710. Omnipress, 2010.

\bibitem[Lowd and Domingos(2005)]{Lowd05naivebayes}
D.~Lowd and P.~Domingos.
\newblock Naive {B}ayes models for probability estimation.
\newblock In \emph{Proceedings of the 22nd International Conference on Machine
  Learning}, pages 529--536. ACM Press, 2005.

\bibitem[Marks and Movellan(2001)]{Marks01diffusionnetworks}
T.~K. Marks and J.~R. Movellan.
\newblock Diffusion networks, products of experts, and factor analysis.
\newblock In \emph{Proc. 3rd Int. Conf. Independent Component Anal. Signal
  Separation}, pages 481--485, 2001.

\bibitem[Mat\'u\v{s} and
  Ay(2003)]{MatusAy03:On_Maximization_of_the_Information_Divergence}
F.~Mat\'u\v{s} and N.~Ay.
\newblock On maximization of the information divergence from an exponential
  family.
\newblock In \emph{Proceedings of the WUPES'03}, pages 199--204. University of
  Economics, Prague, 2003.

\bibitem[Memisevic and Hinton(2010)]{Memisevic2010}
R.~Memisevic and G.~E. Hinton.
\newblock Learning to represent spatial transformations with factored
  higher-order {B}oltzmann machines.
\newblock \emph{Neural Computation}, 22\penalty0 (6):\penalty0 1473--1492, June
  2010.

\bibitem[Mont{\'u}far(2013)]{Montufar2010a}
G.~Mont{\'u}far.
\newblock Mixture decompositions of exponential families using a decomposition
  of their sample spaces.
\newblock \emph{Kybernetika}, 49\penalty0 (1), 2013.

\bibitem[Mont{\'u}far and Ay(2011)]{Montufar2011}
G.~Mont{\'u}far and N.~Ay.
\newblock Refinements of universal approximation results for deep belief
  networks and restricted {B}oltzmann machines.
\newblock \emph{Neural Computation}, 23\penalty0 (5):\penalty0 1306--1319,
  2011.

\bibitem[Mont\'ufar and Morton(2012)]{MontufarMorton2012}
G.~Mont\'ufar and J.~Morton.
\newblock When does a mixture of products contain a product of mixtures?
\newblock \emph{\url{http://arxiv.org/abs/1206.0387}}, 2012.

\bibitem[Mont\'ufar and Morton(2013)]{MontufarMorton2013b}
G.~Mont\'ufar and J.~Morton.
\newblock Geometry of hierarchical models on hidden-visible products of
  simplicial complexes.
\newblock 2013.
\newblock In preparation.

\bibitem[Mont{\'u}far and Rauh(2012)]{MonRauh2012}
G.~Mont{\'u}far and J.~Rauh.
\newblock Scaling of model approximation errors and expected entropy distances.
\newblock In \emph{Proc. of the $9$th Workshop on Uncertainty Processing (WUPES
  2012)}, pages 137--148, 2012.
\newblock Preprint available at \url{http://arxiv.org/abs/1207.3399}.

\bibitem[Mont\'ufar et~al.(2011)Mont\'ufar, Rauh, and Ay]{NIPS2011_0307}
G.~Mont\'ufar, J.~Rauh, and N.~Ay.
\newblock Expressive power and approximation errors of restricted {B}oltzmann
  machines.
\newblock In J.~Shawe-Taylor, R.~Zemel, P.~Bartlett, F.~C.~N. Pereira, and
  K.~Q. Weinberger, editors, \emph{Advances in Neural Information Processing
  Systems 24}, pages 415--423, 2011.

\bibitem[Osindero and Hinton(2008)]{DBLP:conf/nips/OsinderoH07}
S.~Osindero and G.~E. Hinton.
\newblock Modeling image patches with a directed hierarchy of {M}arkov random
  fields.
\newblock In J.~C. Platt, D.~Koller, Y.~Singer, and S.~Roweis, editors,
  \emph{Advances in Neural Information Processing Systems 20}, pages
  1121--1128. MIT Press, Cambridge, MA, 2008.

\bibitem[Pachter and Sturmfels(2004)]{tropical}
L.~Pachter and B.~Sturmfels.
\newblock {Tropical geometry of statistical models}.
\newblock \emph{Proceedings of the National Academy of Sciences of the United
  States of America}, 101\penalty0 (46):\penalty0 16132--16137, Nov. 2004.

\bibitem[Ranzato et~al.(2010)Ranzato, Krizhevsky, and Hinton]{RanzatoKH10}
M.~Ranzato, A.~Krizhevsky, and G.~E. Hinton.
\newblock {Factored 3-Way Restricted Boltzmann Machines For Modeling Natural
  Images}.
\newblock In \emph{Proc. Thirteenth International Conference on Artificial
  Intelligence and Statistics (AISTATS)}, pages 621--628, 2010.

\bibitem[Roth(1934)]{Roth1934}
W.~E. Roth.
\newblock On direct product matrices.
\newblock \emph{Bulletin of the American Mathematical Society}, 40:\penalty0
  461--468, 1934.

\bibitem[Salakhutdinov et~al.(2007)Salakhutdinov, Mnih, and
  Hinton]{Salakhutdinov:2007}
R.~Salakhutdinov, A.~Mnih, and G.~E. Hinton.
\newblock Restricted {B}oltzmann machines for collaborative filtering.
\newblock In \emph{Proceedings of the 24th International Conference on Machine
  Learning}, pages 791--798, 2007.

\bibitem[Sejnowski(1986)]{Sejnowski86higher-orderboltzmann}
T.~J. Sejnowski.
\newblock Higher-order {B}oltzmann machines.
\newblock In \emph{Neural Networks for Computing}, pages 398--403. American
  Institute of Physics, 1986.

\bibitem[Smolensky(1986)]{Smolensky1986}
P.~Smolensky.
\newblock Information processing in dynamical systems: foundations of harmony
  theory.
\newblock In \emph{Symposium on Parallel and Distributed Processing}, 1986.

\bibitem[Tran et~al.(2011)Tran, Phung, and Venkatesh]{mixvarrbm}
T.~Tran, D.~Phung, and S.~Venkatesh.
\newblock Mixed-variate restricted {B}oltzmann machines.
\newblock In \emph{Proc. of 3rd Asian Conference on Machine Learning (ACML)},
  pages 213--229, 2011.

\bibitem[Varshamov(1957)]{Varshamov:1957}
R.~R. Varshamov.
\newblock Estimate of the number of signals in error correcting codes.
\newblock \emph{Doklady Akad. Nauk SSSR}, 117:\penalty0 739--741, 1957.

\bibitem[{Welling} et~al.(2005){Welling}, {Rosen-Zvi}, and
  {Hinton}]{welling:exponential}
M.~{Welling}, M.~{Rosen-Zvi}, and G.~E. {Hinton}.
\newblock Exponential family harmoniums with an application to information
  retrieval.
\newblock In L.~K. Saul, Y.~Weiss, and L.~Bottou, editors, \emph{Advances in
  Neural Information Processing Systems 17}, pages 1481--1488. MIT Press,
  Cambridge, MA, 2005.

\end{thebibliography}

\end{document}